\documentclass{article}

\usepackage{PRIMEarxiv}

\usepackage[utf8]{inputenc} 
\usepackage[T1]{fontenc}    
\usepackage{hyperref}       
\usepackage{url}            
\usepackage{booktabs}       
\usepackage{amsfonts}       
\usepackage{amsmath}
\usepackage{amsthm}
\usepackage{nicefrac}       
\usepackage{microtype}      
\usepackage{lipsum}
\usepackage{fancyhdr}       
\usepackage{graphicx}       
\usepackage{todonotes}
\usepackage{subcaption}
\usepackage{float}
\usepackage{multirow}
\usepackage{longtable}
\graphicspath{{media/}}     

\DeclareMathOperator*{\argmin}{arg\,min}

\newtheorem{corollary}{Corollary}
\newtheorem{definition}{Definition}
\newtheorem{proposition}{Proposition}

\newcommand{\context}[1]{\textcolor{black}{#1}}
\newcommand{\challenge}[1]{\textcolor{black}{#1}}

\newcommand{\proposal}[1]{\textcolor{black}{#1}}
\newcommand{\evaluation}[1]{\textcolor{black}{#1}}
\newcommand{\paperdesc}[1]{\textcolor{black}{#1}}
\newcommand{\revista}[1]{\textcolor{black}{#1}}

\title{Improving {$(\alpha, f)$-}Byzantine Resilience in Federated Learning via layerwise aggregation and cosine distance
}

\author{
  Mario García-Márquez, Nuria Rodríguez-Barroso \\
  Department of Computer Science and Artificial Intelligence,\\ Andalusian Research Institute in Data Science \\
  and Computational Intelligence (DaSCI) \\
  University of Granada \\
  Granada\\
  \texttt{\{mariogmarq, rbnuria\}@ugr.es} \\
   \And
  M.V. Luzón \\
    Department of Software Engineering,\\ 
    Andalusian Research Institute in Data Science \\
  and Computational Intelligence (DaSCI) \\
  University of Granada \\
  Granada\\
  \texttt{luzon@ugr.es} \\
  \And
  Francisco Herrera \\
  Department of Computer Science and Artificial Intelligence,\\ Andalusian Research Institute in Data Science \\
  and Computational Intelligence (DaSCI) \\
  University of Granada \\
  Granada\\
  \texttt{herrera@decsai.ugr.es} \\
}

\begin{document}
\maketitle

\begin{abstract}
\revista{The rapid development of artificial intelligence systems has amplified societal concerns regarding their usage, necessitating regulatory frameworks that encompass data privacy.  Federated Learning (FL) is posed as potential solution to data privacy challenges in distributed machine learning by enabling collaborative model training {without data sharing}. However, FL systems remain vulnerable to Byzantine attacks, where malicious nodes contribute corrupted model updates. While Byzantine Resilient operators have emerged as a widely adopted robust aggregation algorithm to mitigate these attacks, its efficacy diminishes significantly in high-dimensional parameter spaces, sometimes leading to poor performing models. This paper introduces Layerwise Cosine Aggregation, a novel aggregation scheme designed to enhance robustness of these rules in such high-dimensional settings while preserving computational efficiency. A theoretical analysis is presented, demonstrating the superior robustness of the proposed Layerwise Cosine Aggregation compared to original robust aggregation operators. Empirical evaluation across diverse image classification datasets, under varying data distributions and Byzantine attack scenarios, consistently demonstrates the improved performance of Layerwise Cosine Aggregation, achieving up to a 16\% increase in model accuracy.}
\end{abstract}

\keywords{Federated Learning \and Robust Aggregation \and Machine Learning \and Krum}

\section{Introduction}
\context{The rapid proliferation of Artificial Intelligence (AI) technologies has garnered significant societal and regulatory interest. The European Union AI Act~\cite{euact} serves as a prominent example of this regulatory response, underscoring the need to address the application of AI systems, particularly within high-risk domains. Recognizing the still unclear understanding of the broader societal impacts of AI, the concept of trustworthy AI~\cite{trustworthy} has emerged as a critical paradigm. Trustworthy AI is based on the foundational principles of ethics, legality, and technical robustness. Specifically, trustworthy AI is implemented adhering to seven interconnected requirements: (1) human agency and oversight, (2) robustness and safety, (3) privacy and data governance, (4) transparency, (5) diversity and non-discrimination, (6) social and environmental well-being, and (7) accountability.}

\context{Addressing the critical challenges of privacy and data governance, within the framework of regulations such as the GDPR~\cite{gdpr}, Federated Learning (FL)~\cite{mcmahan-2017, tutorialnuria} has gained recognition as a robust and increasingly adopted solution. FL employs a distributed learning mechanism. This mechanism facilitates collaborative model training across multiple clients, avoiding the need for direct data sharing, and consequently preserving client privacy during both training and deployment. However, it is crucial to acknowledge that despite these intrinsic privacy guarantees, FL is still susceptible to adversarial manipulations that can compromise data and model integrity.}

\challenge{Byzantine attacks pose a significant threat to the robustness and security of FL systems. Although considerable academic effort has been directed to address this vulnerability~\cite{surveynuria, zhao23, zukai2025}, a fundamental contribution in this area is the concept of $(\alpha, f)$-Byzantine resilience~\cite{krum}. Based on this foundation, numerous robust aggregation operators have been proposed~\cite{krum, bulyan, nuria22, surveynuria}. However, recent investigations have revealed that several of these operators exhibit susceptibility to attacks, particularly in high-dimensional scenarios~\cite{LittleIsEnough, signguard}. Furthermore, comparative evaluations suggest that the empirical performance of these robust operators is often suboptimal compared to less robust aggregation techniques, such as FedAvg~\cite{krum, bulyan}. This performance gap raises concerns regarding their practical applicability in real-world Federated Learning deployments.}

\proposal{{This study undertakes an investigation into the $(\alpha, f)$-Byzantine resilience properties of well-established aggregation operators, including Krum, Bulyan, and GeoMed~\cite{krum, bulyan}. Furthermore, it introduces a novel aggregation scheme specifically designed to address their limitations when applied to high-dimensional data. High-dimensional settings are often subject to the challenges associated with the curse of dimensionality, where substantial variations in a limited number of coordinates may result in only marginal changes in the overall vector magnitude or distance relative to other vectors. To mitigate this issue, the study proposes leveraging layerwise aggregation to decompose the aggregation problem into more manageable, lower-dimensional sub-problems, facilitating a more efficient and magnitude-sensitive aggregation process. Furthermore, inspired by the demonstrated success of cosine similarity (and consequently, cosine distance) in sparse, high-dimensional contexts~\cite{llama}, this research explores the substitution of Euclidean distance with cosine distance, premised on the rationale that it produces comparable data ordering when applied to normalized data. The proposed Layerwise Cosine aggregation rule scheme, which integrates both layerwise aggregation and cosine distance, demonstrably outperforms the conventional application of baseline operators in terms of both empirical performance and theoretical robustness, while maintaining its core properties. Importantly, this performance enhancement is achieved without introducing additional computational overhead, making the proposal a compelling and resource-efficient alternative to existing baseline operators.}}

\evaluation{To  evaluate the efficacy of the proposed methodology, a comprehensive suite of image classification experiments was conducted utilizing the EMNIST, Fashion-MNIST, CIFAR-10, and CelebA-S datasets. The experimental design incorporated both Independent and Identically Distributed (IID) and non-IID data distributions across the simulated client population. Consistent with the established evaluation paradigm for the robust operators, the primary focus of this assessment was directed towards Byzantine attack scenarios. A comprehensive analysis was performed comparing the original operators, namely Krum~\cite{krum}, Bulyan~\cite{bulyan} and GeoMed~\cite{bulyan}, their result of the Layerwise Cosine Aggregation scheme and its partial applications, aimed at empirically validating the anticipated performance enhancements. The empirical results show an improvement with respect to the baseline operator in every scenario showing up to 16\% improved accuracy.}

\evaluation{To  assess the effectiveness of the proposed methodology, a comprehensive set of image classification experiments was performed, employing the EMNIST, Fashion-MNIST, CIFAR-10, and CelebA-S datasets. The experimental design encompassed both Independent and Identically Distributed (IID) and non-IID data distributions across the simulated client population. In accordance with the established evaluation paradigm for robust aggregation operators, the primary emphasis of this assessment was placed on Byzantine attack scenarios. A detailed comparative analysis was performed that evaluated the performance of the original operators, namely Krum~\cite{krum}, Bulyan~\cite{bulyan} and GeoMed~\cite{bulyan}, along with the proposed layer-wise cosine aggregate scheme and its partial applications. This analysis aimed to empirically validate the anticipated performance gains. The empirical results demonstrate a consistent improvement over the baseline operators in all evaluated scenarios, with accuracy gains of up to 16\% observed.}

\paperdesc{The subsequent sections of this paper are organized as follows. Section~\ref{sec:background} provides the necessary background knowledge to follow the rest of the paper, including formal definitions of FL (Section~\ref{sec:fl}), Byzantine attacks (Section~\ref{sec:byzantine}), the notion of $(\alpha,f)$ -Byzantine resilience (Section~\ref{sec:krum}), and some of the most popular robust operators (Section~\ref{sec:rar}). The proposed aggregation scheme and its theoretical properties are presented in Section~\ref{sec:proposal}. Section~\ref{sec:expsetup} details the experiments performed, including a description of the datasets and models (Section~\ref{sec:datasets}), a description of the Byzantine attacks implemented (Section~\ref{sec:labelflipping}), the baselines used for comparison (Section~\ref{sec:baselines}) and additional relevant implementation details (Section~\ref{sec:details}). Section~\ref{sec:analysis} provides a comparative analysis of the experimental results, considering two distinct cases: the absence of adversarial clients (Section~\ref{sec:noattack}) and the presence of adversarial clients (Section~\ref{sec:underattack}). Finally, concluding remarks are presented in Section~\ref{sec:conclusions}.}

\section{Background}\label{sec:background}
This section aims to provide the background required to follow the rest of the work.

\subsection{Federated Learning}\label{sec:fl}
FL is a distributed machine learning paradigm in which multiple entities collaborate to train a global model without explicitly exchanging data, preserving the privacy of each entity~\cite{mcmahan-2017}. FL operates in two distinct phases:
    \begin{enumerate}
        \item \textbf{Training Phase.} In this phase, each client shares information derived from their local data, without revealing the raw dataset itself, to jointly train a machine learning model. This resulting global model may be stored on a single server, a designated client, or distributed among the clients.
        \item \textbf{Inference Phase.} In this phase, clients collaborate to apply the collaboratively learned model to new instances.
    \end{enumerate}
Both phases can be performed synchronously or asynchronously, depending on various factors such as client availability. Formally~\cite{tutorialnuria}, an FL scenario can be modeled as follows. We consider a set of $n$ clients or data owners, denoted by $\{C_1, \dots, C_n\}$, where each client $C_i$ possesses a local dataset $D_i$. Each client $C_i$ also maintains a local model $L_i$, parameterized by weights $V_i \in \mathbb{R}^d$. The primary objective of FL is to train a global model $G$ by leveraging the local datasets of the clients through an iterative process, to which each iteration is called a round of learning.

In a given round $t$, each client trains its local model using its local dataset $D_i^t$, resulting in a local model $L_i^t$ with updated parameters $V_i^t$. These parameters $V_i^t$ are then shared with a central parameter server to compute the global model $\mathcal{G}^{t+1}$. The global model{'s parameters are} computed from aggregating the distinct local parameters $V_1^t, \dots, V_n^t$ using a predefined aggregation function $\Delta$. Formally, the global model{'s parameters $V_{\mathcal{G}}^{t+1}$} are computed as:
\begin{equation}
    {V_{\mathcal{G}}}^{t+1} = \Delta(V_1^t, \dots, V_n^t).
    \label{eq:global_model_update}
\end{equation}
Subsequently, the local model $L_i^{t+1}$ is set to the global model ${\mathcal{G}}^{t+1}$ for every client $i$. This process is repeated until a predefined stopping criterion is met. Upon termination, the global model ${\mathcal{G}}$ encapsulates the knowledge learned from all individual clients.

\subsection{Byzantine Attacks}\label{sec:byzantine}
FL, as a specialized instance of machine learning, inherits the susceptibility of its parent field to adversarial attacks that seek to degrade performance or compromise privacy. The existing literature categorizes these attacks through several criteria~\cite{surveynuria}, including the attacker's knowledge of the system, the manipulation of model behavior, and the attack's objective. Within the latter taxonomy, untargeted attacks aim solely to reduce the model's performance on the primary learning task. A particularly challenging case of this attack is the Byzantine attack~\cite{byzantinegenerals, Hu2021ChallengesAA}, in which a subset of clients submit arbitrary updates. These updates are typically generated randomly or derived from models trained on manipulated data, effectively producing random updates.

These attacks are frequently used in conjunction with model replacement techniques~\cite{howtobackdoor}. This is due to the fluctuating proportion of adversarial clients, which can prevent the mitigation of malicious updates, as the sheer number of benign client updates may not effectively counteract the influence of compromised updates.
\subsection{Byzantine resilience}\label{sec:krum}
The prevalence of Byzantine attacks has motivated extensive research to improve the resilience of FL systems against these malicious intrusions. Many studies concentrate on the deployment of defensive strategies against Byzantine attacks on the server, particularly within the aggregation operator used to compute the parameters of the global model~\cite{surveynuria}. A fundamental question is to establish the conditions under which an aggregation operator or rule can be considered robust against Byzantine attacks.  In this context, the most widely used definition is that of a $(\alpha, f)$-Byzantine Resilient rule~\cite{krum}, detailed in Definition~\ref{def:alphafbyz}.

\begin{definition}\label{def:alphafbyz}
    Let $\alpha \in [0, \frac{\pi}{2}]$ be any angle and $f \in \{1, \ldots, n-1 \}$ any integer. Let $n \in \mathbb{N}$, $V_i \in \mathbb{R}^d$ with $1 \le i \le n-f$ be independent, identically distributed vectors, with $V_i \sim G$ with $\mathbb{E}G=g$. Let $B_k \in \mathbb{R}^d$ with $n-f+1 \le k \le n$ be random vectors, possibly dependent between them and the vectors $V_i's$. An aggregation rule $\mathcal{F}$ is said to be $(\alpha, f)$-Byzantine resilient if, for any $1 \le j_1 \le \ldots \le j_f \le n$, the vector:
    \begin{equation*}
        F = \mathcal{F}(V_1, \ldots, B_1, \ldots, B_f, \ldots, V_n)
\end{equation*} satisfies (i) $\langle \mathbb{E}F, g\rangle \ge (1 - \sin \alpha)||g||^2$ and (ii) for $r \in \{2, 3, 4\}$, $\mathbb{E}||F||^r$ is bounded by linear combinations of terms of the form $\mathbb{E}||G||^{r_1} \cdot \ldots \cdot ||G||^{r_{n-1}}$ with $r_1 \cdot \ldots \cdot r_{n-1}=r$.
\end{definition}

We will not dive into the theoretical details of this definition, but one can intuitively see it that given $f$ Byzantine vectors, the expected output of the rule should deviate no more than an angle $\alpha$ from the true gradient. The second condition stated in the definition is required to transfer the dynamics of convergence of SGD to the operator~\cite{krum, eon1998}. Also, note that $\alpha$ and $f$ may vary between rules, thus giving us also a sense of how robust a rule may be compared to other.

\subsection{Robust Aggregation Rules}\label{sec:rar}
The notion of $(\alpha, f)$-Byzantine Robustness has been widely adopted in the literature as a basis for developing new robust aggregation rules. In particular, the Krum aggregation operator~\cite{krum} stands out as the rule most frequently encountered in this body of work, and importantly, it was the operator that originally established the definition of $(\alpha, f)$-Byzantine Robustness. The Krum operator, $KR$, is defined as:

\begin{equation*}
    KR(V_1, \ldots, V_n) = \argmin_{i \in \{1, \ldots, n\}} s(i)
\end{equation*}

where $s(i)$ is defined as
\begin{equation*}
    s(i) = \sum_{i \to j} || V_i - V_j ||^2.
\end{equation*}

Here, $f$ is an hyperparameter to be taken into account, which denotes the expected amount of Byzantine vectors in the updates set. The original paper~\cite{krum} proves the theoretical Byzantine resistance when $f < \frac{n}{2} - 1$.

Other common rules fundamented upon this definition include Bulyan~\cite{bulyan} and GeoMed~\cite{bulyan, rousseeuw85}. Bulyan is a robust aggregation rule designed to enhance resilience against Byzantine attacks in distributed systems and federated learning. It operates through a two-step process: first, a subset of updates is selected using Krum, favoring those least affected by outliers or malicious contributions. Then, a coordinate-wise trimmed mean is applied to the selected updates, eliminating extreme values to produce a robust aggregated result. GeoMed is an aggregation rule inspired by the Geometric Median~\cite{rousseeuw85}. To address the high computational cost of the Geometric Median, GeoMed instead selects the medoid of the updates. The medoid is the point within a dataset that minimizes the sum of distances to all other data points. Although not formally proven to be an $(\alpha, f)$-Byzantine resilient rule, GeoMed is considered a strong candidate and is frequently treated as such in the literature~\cite{bulyan}.

\section{{Layerwise Cosine aggregation rule}}\label{sec:proposal}
This section introduces the Layerwise Cosine {aggregation rule} concept, our proposed approach to improve $(\alpha, f)$-Byzantine Resilient Rules while sustaining their robustness.  We begin by outlining the limitations of specific operators in Section~\ref{sec:limitations}. Following this, we detail our proposal and perform a rigorous theoretical analysis of the Byzantine resilience of the Layerwise Cosine aggregation rule scheme in Section~\ref{sec:resilence}.

\subsection{Limitations of operators}\label{sec:limitations}
Byzantine resilient operators, such as Krum or Bulyan, are widely employed in FL due to their ability to enhance resilience against Byzantine attacks~\cite{surveynuria}. However, it is increasingly acknowledged that these methods can suffer from a reduced performance compared to alternatives such as FedAvg, especially when applied to high-dimensional models~\cite{krum}. Recent investigations~\cite{LittleIsEnough} have specifically highlighted the vulnerability of Krum, and, by extension, derived approaches such as Bulyan, to attacks that exploit high dimensionality. In these scenarios, operators may struggle to effectively identify subtle but critical adversarial manipulations spread across many dimensions.

\begin{figure}[ht]
    \begin{subfigure}{0.66\linewidth}
        \centering
        \includegraphics[width=\linewidth]{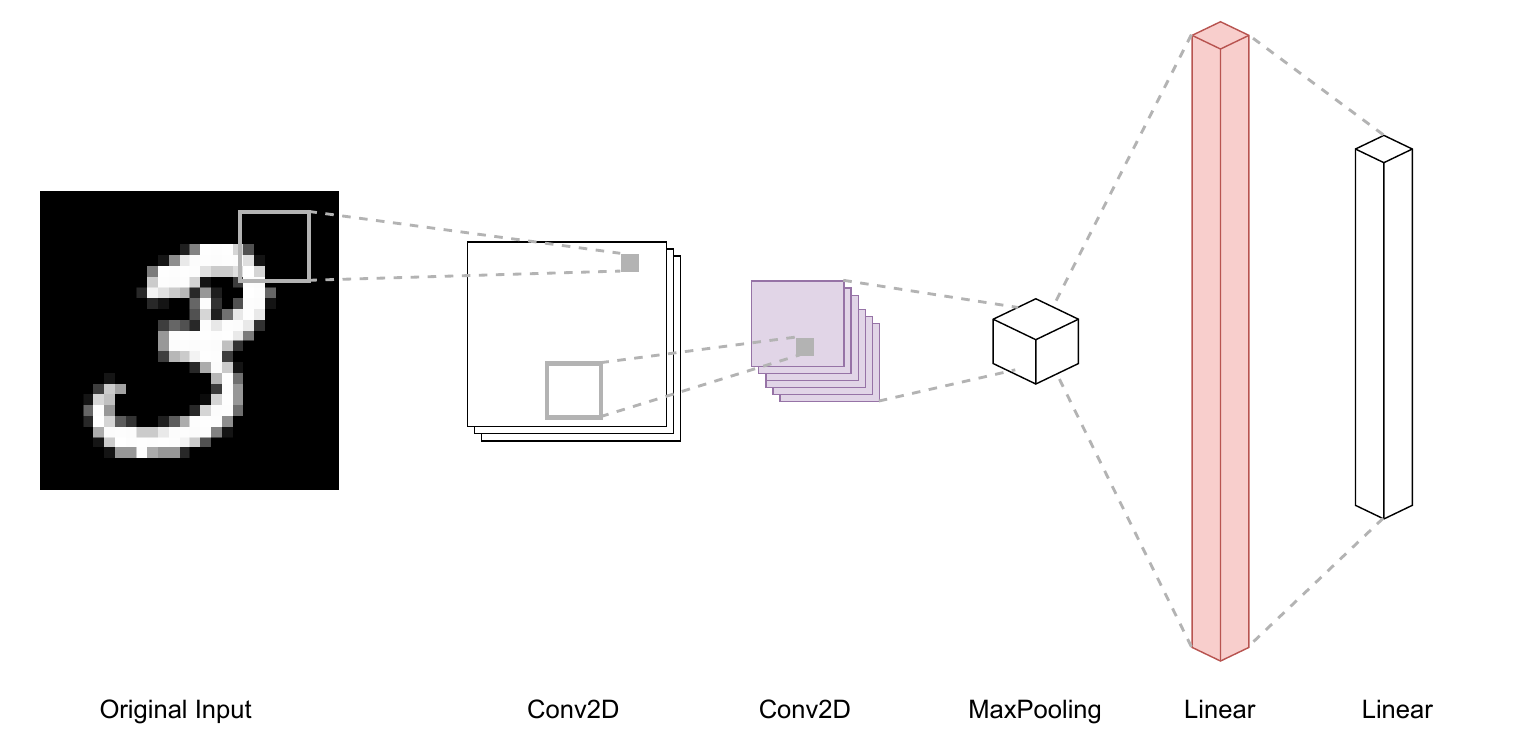}
        \caption{Neural Network Diagram.}
    \end{subfigure}
    \begin{subfigure}{0.33\linewidth}
        \centering
        \includegraphics[width=\linewidth]{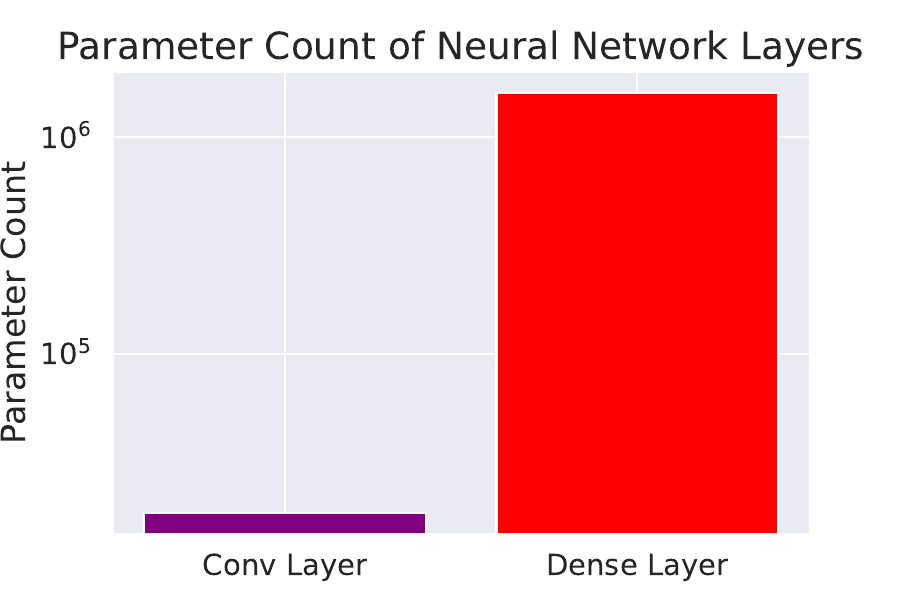}
        \caption{Histogram for parameters in layers.}
    \end{subfigure}
    \caption{{Parameter Distribution across Layers in a Two-Layer CNN used in some of the experiments. This figure highlights the significant parameter imbalance inherent in a basic two-layer CNN.  A histogram directly compares the parameter count of the second convolutional layer (purple) and the first dense layer (red), revealing an order of magnitude difference.  This disparity underscores the typical parameter distribution imbalance in such architectures.  Detailed architectural specifications are provided in Section~\ref{sec:details}.}}
    \label{fig:imbalance}
\end{figure}

A clear example of the challenges facing these aggregation operators arises in the context of aggregation of computational neural networks (CNNs). CNNs naturally exhibit an imbalance in the number of parameters across different layers, with convolutional layers typically having fewer parameters than dense layers. This parameter imbalance can distort the weighting process in algorithms for these operators. The result can be a misalignment between the actual importance of each layer to the network's overall performance and the weight assigned by the aggregation operator, which tends to be skewed by the sheer number of parameters in a layer. {A visual example of this problem can be seen in Figure~\ref{fig:imbalance}}.

Furthermore, even when dealing with architectures that do not exhibit significant parameter imbalances across layers, such as within Dense layers or Multilayer Perceptrons, aggregation rules like Krum and GeoMed often operate on data that exhibit sparsity. This sparsity introduces a notable difficulty because Euclidean distance, the fundamental metric used in Krum's selection and conceptually related to GeoMed's medoid choice, is not ideally suited for capturing meaningful distances within sparse data representations~\cite{llama}.

\subsection{\revista{Layerwise Cosine Aggregation}}\label{sec:resilence}
In this section we will introduce our proposal, Layerwise Cosine Aggregation, a framework for improving robust aggregation rules, and provide theoretical results that our proposal is $(\alpha, f)$-Byzantine Resilient under the same conditions as the original rules. This result indicates that the improvements made to any Byzantine Resilient operator lead to a valid robust aggregation rule for Byzantine learning.


Proposition \ref{th:layerwise} formally proves that the layerwise application of \revista{an $(\alpha, f)$-Byzantine Resilient rule}, also satisfies the criteria for Byzantine Resilience. From the proof of this proposition we infer that the bound shown in condition (i) of Definition~\ref{def:alphafbyz} is improved over the original rule.

Subsequently, we explain how replacing the Euclidean distance within a given operator with the cosine distance, in conjunction with median gradient clipping, also results in a $(\alpha, f)$-Byzantine resilient rule.

Therefore, to address the limitations mentioned above, we propose a novel approach, \textbf{Layerwise Cosine Aggregation}, which uses layer-wise aggregation to enhance the robustness of the given operator \revista{in conjunction with cosine distance and median gradient clipping}. As illustrated in Figure~\ref{fig:layerwise}, this approach decomposes the aggregation process into a series of layerwise sub-problems, effectively reducing the dimensionality of the parameter space considered by the original operator. By adopting a layer-wise strategy, we ensure that the robust operator accurately weights each layer based on its true contribution to model utility, rather than being influenced by the number of parameters within that layer.

\begin{figure*}[ht]
    \centering
    \includegraphics[width=0.7\linewidth]{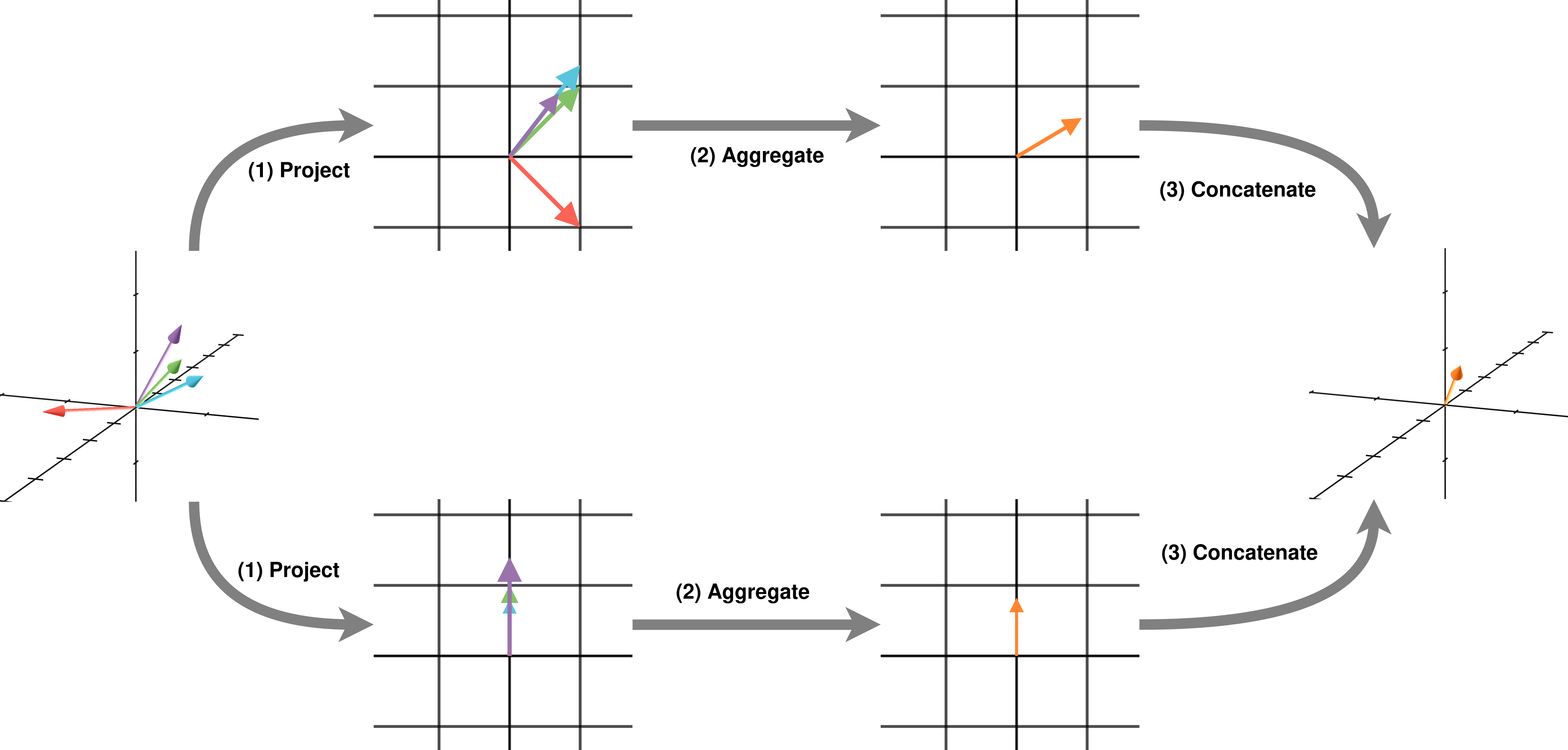}
    \caption{In a layerwise aggregation approach, original vectors are projected into orthogonal subspaces that partition the original space. Each set of projected vectors is then aggregated, and the resulting aggregated vectors are concatenated to produce a vector in the original space.}
    \label{fig:layerwise}
\end{figure*}

\revista{Furthermore, even with the decomposition of the problem into subproblems, these subproblems may still present challenges for the operator. This difficulty stems from the significant imbalance between the number of updates and the parameter count, leading to a scenario where updates populate only a small portion of the parameter space. The cosine distance (and cosine similarity) has demonstrated effectiveness in such circumstances~\cite{llama}. However, given that the cosine distance does not consider the magnitudes of updates, we incorporate median gradient clipping to address this limitation. The clipping to the median of the norms is employed as it consistently resides within the benign set~\cite{LittleIsEnough}. In particular, these enhancements introduce minimal additional computational overhead, requiring approximately equivalent memory and computational resources as the original chosen operator.}


First, we will formally define the layer-wise application of a rule. Let $d=m_1 + \ldots + m_k$. We will call $k$ the numbers of layers. Given an aggregation rule $\mathcal{F}$, we define the layerwise application of $\mathcal{F}$, denoted by $L\mathcal{F}$, as
\begin{equation*}
    L\mathcal{F}(V_1, \ldots, V_n) = (\mathcal{F}(V_{1,1}, \ldots, V_{n,1}), \ldots, \mathcal{F}(V_{1,k}, \ldots, V_{n,k}))
\end{equation*}
where $V_{i,j}$ denotes the projection of the vector $V_i$ to the $j$-th element of the partition $R^{m_1} \times \ldots \times R^{m_k}$ of $R^d$.

\begin{proposition}\label{th:layerwise}
    The layerwise application of an $(\alpha, f)$-Byzantine Resilient rule $\mathcal{F}$, $L\mathcal{F}$, is also an $(\alpha, f)$-Byzantine Resilient rule.
\end{proposition}

\begin{proof}
    We will proceed by induction over the number of layers. For the first case, where the number of layers is 1, the result is direct by the fact that $\mathcal{F}$ is $(\alpha, f)$-Byzantine Resilient rule. Moving to the case with $n+1$ layers. Let $\mathbb{E}G = g \in \mathbb{R}^{h+m}$ where $h$ is the dimensionality of the parameters of the $n$ first layers and $m$ of the $n+1$-th layer. Let $g_1$ denote the projection of $g$ into its first $h$ coordinates and $g_2$ into his last $m$ coordinates. Since $\mathcal{F}$ is $(\alpha_m, f)$-Byzantine Resilient and $L\mathcal{F}$ is $(\alpha, f)$-Byzantine Resilient for the case of $n$ layers, we have:
    \begin{equation*}
        \begin{split}
        L\mathcal{F}(V_1, \ldots, V_n) & = (L\mathcal{F}(V_{1, 1}, \ldots, V_{1, 1}), \mathcal{F}(V_{1,2}, \ldots, V_{n,2})) \\
        & = (F_1, F_2) = F
        \end{split}
    \end{equation*}
    with the partition of $R^d=R^h\times R^m$. Then:
    \begin{equation*}
        \langle \mathbb{E}F_1, g_1 \rangle \ge (1 - \sin \alpha_h) || g_1 ||^2.
    \end{equation*}
    For $F_2$ and $g_2$ one should change $\alpha_h$ by $\alpha_m$. Considering $F=(F_1, F_2) \in \mathbb{R}^{h+m}$, we have

    \begin{equation*}
        \begin{split}
        \langle \mathbb{E}F, g \rangle = \langle \mathbb{E}F_1, g_1 \rangle + \langle \mathbb{E}F_2, g_2 \rangle & \ge (1 - sin \alpha)(||g_1||^2 + ||g_2||^2) \\
        &= (1 - sin \alpha)||g||^2
        \end{split}
    \end{equation*}
    where $\alpha=\max \{ \alpha_h, \alpha_m\}$.

    The second property can be render equivalent to the property $\mathbb{E}||\mathcal{F}||^r \le B_r + A_r \mathbb{E}||G||^r$ where $A_r$ and $B_r$ are scalars, with $r=2,3,4$\cite{eon1998, bulyan, krum}. We will proceed again by induction. The case of only one layer is again direct by the property that $\mathcal{F}$ is $(\alpha, f)$-Byzantine Resilient. Then, moving to the case of $n+1$ and keeping notation of the previous part we have:
    \begin{equation*}
    \begin{split}
        \mathbb{E}||F||^r & \le \mathbb{E}||F_1||^r + \mathbb{E}||F_2||^r \\
        & \le (B_1 + A_1\mathbb{E}||g_1||^r) + (B_2 + A_2\mathbb{E}||g_2||^r) \\
        & \le (B_1 + A_1\mathbb{E}||G||^r) + (B_2 + A_2\mathbb{E}||G||^r) \\
        & = (B_1 + B_2) + (A_1 + A_2)\mathbb{E}||G||^r.
    \end{split}
    \end{equation*}
\end{proof}

Now, we remember that the cosine distance between two vector $V_j$ and $V_i$ is defined as the $1 - \cos \theta_{ij}$ where $\theta_{ij}$ denotes the angle form by the two vectors. A well-known result indicates that for normalized data points, ($||V_i||=||V_j||$ for every $i$ and $j$), similar to the data that we obtain after clipping, Euclidean and cosine distances exhibit a monotonic relationship~\cite{quian04}. This monotonicity implies that within normalized data, algorithms that use Euclidean distance to rank data points, for example by proximity, will produce equivalent rankings if cosine distance is used instead. Consequently, for distance-based aggregation operators like Krum, Bulyan, or GeoMed, replacing Euclidean distance with cosine distance should maintain their fundamental behavior and $(\alpha, f)$-Byzantine Resilience. Despite this theoretical expectation of equivalence, it is plausible to anticipate empirical gains from employing cosine distance, as suggested by its successful application in areas such as vector embedding similarity~\cite{llama} due to the sparsity of the update space.

\begin{corollary}\label{th:layercos}
    \revista{The Layerwise Cosine Aggregation of an $(\alpha, f)$-Byzantine Resilient rule which uses euclidean distance for ranking points, is also $(\alpha, f)$-Byzantine Resilient.}
\end{corollary}
\begin{proof}
    \revista{The proof is direct by applying Proposition \ref{th:layerwise} and taking into account the last paragraph.}
\end{proof}

An important result is that the Layerwise and Layerwise Cosine application of an operator may provide a better theoretical bound for the angle $\alpha$ than the original operator. Given $(\alpha, f)$-byzantine resilient rule $\mathcal{F}$, its layer-wise application $L\mathcal{F}$ given a layer partition $d=m_1 + \ldots + m_k$, takes $\alpha$ as $\max\{\alpha_{m_1}, \ldots, \alpha_{m_k}\}$ where $\alpha_{i}$ is the $\alpha$ when $\mathcal{F}$ is applied to data with dimension $i$. This result is derived from the proof of the byzantine robustness of $L\mathcal{F}$. In the case of Krum (and Bulyan), $\sin\alpha$ is directly proportional to the square root of the data dimension, and thus we are able to maximize the bound of the angle, providing a better theoretical robustness.

\section{Experimental Setup}\label{sec:expsetup}
This section details the experimental setup designed to evaluate the robustness and effectiveness of the Layerwise Cosine Aggregation scheme. To evaluate the proposed method, we employed image classification models trained in various FL-optimized datasets. This section provides the necessary information to replicate our experimental environment, including the following. Section \ref{sec:datasets} outlines the datasets and models employed in each dataset, Section \ref{sec:labelflipping} introduces the label flipping attack, a Byzantine attack used in federated settings, Section \ref{sec:baselines} explains the baselines used in the experiments, and Section \ref{sec:details} provides specific implementation details.  Our experimental design encompasses two distinct scenarios: evaluations conducted in the absence of adversarial clients and evaluations performed within a federated setting deliberately incorporating adversarial clients.

\subsection{Dataset and models}\label{sec:datasets}
For the evaluation of our proposal we have used classical image classification datasets. The datasets used and their federated distributions used are described as follows:

\begin{itemize}
    \item \textbf{CIFAR-10}. This dataset is a labeled subset of the famous 80 million images dataset~\cite{cifar10}, consisting of 60,000 32x32 color images in 10 classes. The training data is evenly distributed between 200 clients.
    \item \textbf{Fashion MNIST}. This dataset~\cite{fashionmnist-2017} aims to be a more challenging replacement for the original MNIST dataset. It contains 28x28 grayscale clothing images from 10 different classes. We set the number of clients to 200.
    \item \textbf{EMNIST Non-IID}. This dataset, presented in 2017 in \cite{emnist}, is an extension of the MNIST dataset~\cite{lecun-1998}. The EMNIST Digits class contains a balanced subset of the digits dataset. This dataset has been federated so that each client corresponds to the digits written by the same author, thus providing a real Non-IID federated setting.
    \item \textbf{EMNIST}. This dataset consists of the EMNIST data set, but where the training data have been federated following an IID distribution between 200 clients.
    \item \textbf{Celeba-S Non-IID}. The CelebA~\cite{celeba} dataset consists of famous face images with 40 binary attribute annotations per image. We use it as a binary image classification dataset, selecting a specific attribute as target, in particular, \textit{Smiling}. We federate the dataset by assigning each famous person a client. Thus, the number of clients is set to more than 8000, but since some clients have a very poor dataset, we only considered clients with more than 30 samples in their local dataset, making the number of clients considered to be 1800.
    \item \textbf{Celeba-S}. This dataset consists on the Celeba-S dataset described above but where the data has been evenly distributed between 200 clients.
\end{itemize}

Regarding the models, we construct a convolutional neural network with 2 convolutional layers and 2 fully-connected layers as the global model for the Fashion MNIST, EMNIST and EMNIST Non-IID datasets. For the CIFAR-10, Celeba-S and Celeba-S Non-IID datasets, we use a pre-trained \texttt{EffectiveNet-B0}~\cite{efficientnet-2019} as the global model. Both models have been trained using an Adam optimizer with a learning rate of 0.001 and 10 epochs per round per client.

\subsection{Label Flipping Byzantine attacks}\label{sec:labelflipping}
For the poisoning attack, we consider a label-flipping-based attack. In this attack, the labels of a randomly selected subset of training data from adversarial clients are modified to contain a random label. This deliberate mislabeling causes adversarial clients to generate virtually random model updates, inducing a degradation in the performance of the global model. 

This attack is often implemented in conjunction with a model replacement technique designed to amplify the impact of adversarial updates. Given a model parameter aggregation rule defined by the following equation:

\begin{equation}\label{eq:updaterule}
    {V_\mathcal{G}}^{t+1}={V_\mathcal{G}}^{t}+ \frac{\eta}{n}\sum_{i=1}^n({V}_i^t - {V_\mathcal{G}}^t)
\end{equation}

where $\eta$ represents the server learning rate (assumed to be one for the remainder of this work), an adversarial client may transmit the following update:

\begin{equation}\label{eq:boost}
    \hat{{V}}_{adv}^t = \beta({V}_{adv}^t - {V_\mathcal{G}}^t)
\end{equation}

where $\beta = \frac{n}{\eta}$. Combining Equations \ref{eq:updaterule} and \ref{eq:boost}, and assuming convergence of benign clients, yields the following approximation.

\begin{equation*}
    {V_\mathcal{G}}^{t+1} \approx {V_\mathcal{G}}^{t} + \frac{\eta}{n} \frac{n}{\eta} ({V}_{adv}^t - {V_\mathcal{G}}^{t}) = {V}_{adv}^t.
\end{equation*}

\subsection{{Baselines}}\label{sec:baselines}
To empirically assess the efficacy of the Layerwise Cosine aggregation rule, we establish a set of baseline aggregation techniques for comparative analysis. We selected three well-established robust aggregation rules as baselines, namely Krum, Bulyan, and GeoMed, which are detailed in the following.

\begin{itemize}
    \item \textbf{Krum}~\cite{krum}: The Krum aggregation rule selects the update from the set of updates submitted $\{V_1, \ldots, V_n\}$ that minimizes the sum of squared Euclidean distances to a subset of other updates. Formally, the Krum operator $KR$ is defined as:
    \begin{equation*}
        KR(V_1, \ldots, V_n) = \underset{i \in \{1, \ldots, n\}}{\operatorname{argmin}} \ s(i)
    \end{equation*}
    where the score function $s(i)$ is given by:
    \begin{equation*}
        s(i) = \sum_{i \to j} \| V_i - V_j \|^2
    \end{equation*}

    \item \textbf{Bulyan}~\cite{bulyan}: The Bulyan aggregation operator is based on the Krum selection process. It first computes the Krum score $s(i)$ for each update $V_i$. Subsequently, it selects a subset $\mathcal{V}$ of $m$ updates with the lowest Krum scores. The aggregated update is then computed as the average of the updates in this selected subset $\mathcal{V}$. Formally, the Bulyan operator $B$ is defined as:
    \begin{equation*}
        B(V_1, \ldots, V_n) = \frac{1}{|\mathcal{V}|}\sum_{V \in \mathcal{V}} V
    \end{equation*}
    where $\mathcal{V} \subseteq \{V_1, \ldots, V_n\}$ is the set of $m$ updates with the lowest scores $s(i)$ as defined by the Krum operator. Following common practice, we set the hyperparameter $m$ to 5 for our experiments.

    \item \textbf{GeoMed}~\cite{bulyan,rousseeuw85}: The Geometric Median (GeoMed) operator identifies the update within the set $\{V_1, \ldots, V_n\}$ that minimizes the sum of squared Euclidean distances to all other updates in the set.  The GeoMed operator $Geo$ is defined as:
    \begin{equation*}
        Geo(V_1, \ldots, V_n) = \underset{V_j \in \{V_1, \ldots, V_n\}}{\operatorname{argmin}} \sum_{i =1}^n \| V_i - V_j \|^2
    \end{equation*}
    Essentially, GeoMed selects the update that is centrally located with respect to the other submitted updates in terms of Euclidean distance.
\end{itemize}

For each of these baseline operators, we consider several variations to provide a comprehensive comparison. These variations include:
\begin{itemize}
    \item \textbf{Original Operator}: The direct application of the standard Krum, Bulyan, or GeoMed operator as defined above.
    \item \textbf{Layerwise Application}:  Applying each operator in a layer-wise manner. In this approach, the aggregation is performed independently for each layer of the neural network architecture utilized in the experiments.
    \item \textbf{Cosine Distance with Median Clipping}:  Modifying the original operator by substituting the Euclidean distance metric with the cosine distance. Additionally, we incorporate median gradient clipping as a preprocessing step before aggregation when employing cosine distance, to further enhance robustness and maintain correct ordering.
    \item \textbf{Full Proposal (Layerwise Cosine)}:  The complete proposed scheme, which integrates both layerwise aggregation and the use of cosine distance in conjunction with median gradient clipping, applied to the chosen baseline operator.
\end{itemize}
This comprehensive set of baselines and their variations allows for a rigorous empirical evaluation of the proposed aggregation scheme.

\subsection{Implementation details}\label{sec:details}
To ensure reproducibility, we provide the code used to run all experiments~\footnote{\url{https://github.com/ari-dasci/S-layerwise_cosine_aggregation}}. The code is written using the FLEXible FL framework~\cite{herrera2024flex}~\footnote{\url{https://github.com/FLEXible-FL/FLEXible}} and its companion library \texttt{flex-clash}~\footnote{\url{https://github.com/FLEXible-FL/flex-clash}} to simulate attacks on the FL scheme and access basic federated aggregation operators; also PyTorch~\cite{pytorch} has been used for implementing the models.

\section{Analysis of Results}\label{sec:analysis}
In this section, we present a comprehensive analysis of the experimental results obtained to demonstrate the validity and robustness of the rules derived from the Layerwise Cosine Aggregation scheme as aggregation mechanisms. We will assess the effectiveness of the proposed method, focusing on test loss and accuracy in image classification models trained in various FL-adapted datasets. This analysis serves to validate the superiority of our Layerwise Cosine Aggregation proposal. We pay particular attention to the loss, examining its equivalence of results and its unbounded nature, which facilitates clearer analysis. The analysis is structured around two distinct scenarios explored in our experiments: the results obtained when no adversarial clients are present, analyzed in Section \ref{sec:noattack}, and the results of federated settings under attack, explored in Section \ref{sec:underattack}.

\subsection{Analysis under no attack}\label{sec:noattack}
The main results of the final test loss and accuracy achieved in the training under different operator variations are displayed in Table \ref{tab:loss}, Table \ref{tab:loss_geomed} and Table \ref{tab:loss_bulyan}, which displays the last test loss, the average of the test loss in the last 10 rounds, the minimum test loss achieved during the training, and the equivalent metrics for the accuracy for Krum, GeoMed and Bulyan, respectively. The results of the original operators are used as the baseline. However, to visualize the training process, the speed of convergence and mitigate the impact of potential overfitting, Figures \ref{fig:loss_no_attack}, \ref{fig:loss_no_attack_geomed}, and \ref{fig:loss_no_attack_bulyan} illustrate the evolution of test loss during training rounds in relevant cases.

\revista{Across all six datasets, our proposed Layerwise Cosine Aggregation method consistently outperforms the standard operators, demonstrating improvements across all evaluated metrics. For example, on the Celeba-S Non-IID dataset, the Layerwise Cosine Krum achieves a substantial accuracy improvement exceeding 16\%, while on CIFAR-10 the Layerwise Cosine Bulyan obtains an improvement over 12\%. Furthermore, Figures \ref{fig:loss_no_attack} and \ref{fig:loss_no_attack_geomed} highlight a key advantage of our approach, particularly evident on the EMNIST Non-IID dataset. Here, the standard aggregators exhibit overfitting, characterized by an increasing test loss over communication rounds. In contrast, layerwise cosine aggregation and its partial applications effectively mitigate this overfitting. We attribute this superior performance to the significant parameter imbalance within the CNN model employed. Specifically, the first dense layer comprises 1.6 million parameters, while the initial convolutional layer contains only approximately 600. The potential overemphasis of the original operator on the dense layer may lead to model selection based primarily on the distribution of parameters within this layer, which may not be globally optimal. Layer-wise aggregation, by selectively aggregating updates for each layer, addresses this issue and yields improved results. Moreover, when selecting updates within the high-dimensional and sparse parameter space of the dense layer (100 samples in a space exceeding 1.6 million parameters), the cosine distance metric proves to be more effective than the Euclidean distance for comparing updates. Finally, our combined method, leveraging both layer-wise aggregation and cosine similarity, improves each individual technique, achieving the best overall performance.}

\revista{This performance trend is consistent across the distinct datasets and operators. In cases with extreme parameter imbalance, such as the aforementioned example, Cosine Aggregation outperforms Layerwise aggregation. This likely occurs because even with layer separation, the dimensionality of the dense layer remains substantial, echoing the challenges faced by standard operators. Consequently, cosine distance proves to be more effective than layer-wise aggregation in such scenarios. In contrast, layer-wise aggregation emerges as a stronger option when using the \texttt{EffectiveNet-B0} architecture. This network, while having a larger total number of parameters, distributes them across several layers. Therefore, layer-wise aggregation effectively reduces the dimensionality of each subproblem, improving performance compared to relying solely on cosine similarity. These results suggest that neither approach is universally superior. Rather, the combination of both techniques in Layerwise Cosine aggregation leverages the strengths of each, leading to impressive performance across diverse scenarios.}

A notable exception to this general trend is observed in the case of Bulyan. As detailed in Table~\ref{tab:loss_bulyan}, the layer-wise aggregation exhibits a performance comparable to that of standard aggregation. In contrast, cosine aggregation demonstrates a significant improvement in all scenarios evaluated. We hypothesize that the enhanced performance of cosine aggregation can be attributed to its unique characteristic among the operators considered: it is the only method that employs multiple local updates in the computation of the global parameter. This increased information utilization may lead to a more balanced contribution from each layer, thereby replicating the performance advantages observed in layerwise aggregation.
\begin{figure}[ht]
    \begin{subfigure}{0.30\linewidth}
        \centering
        \includegraphics[width=\linewidth]{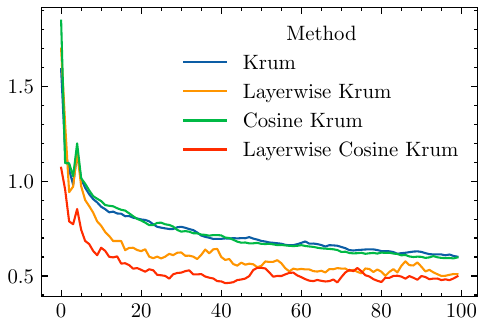}
        \caption{CIFAR-10 dataset.}
    \end{subfigure}
    \begin{subfigure}{0.30\linewidth}
        \centering
        \includegraphics[width=\linewidth]{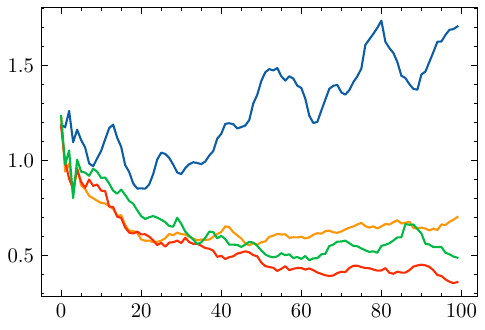}
        \caption{EMNIST Non-IID dataset.}
    \end{subfigure}
    \begin{subfigure}{0.30\linewidth}
        \centering
        \includegraphics[width=\linewidth]{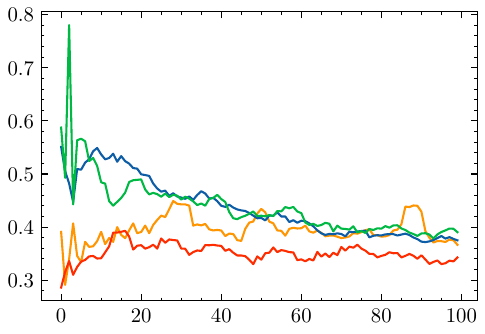}
        \caption{CelebA-S dataset.}
    \end{subfigure}
    \caption{Test Loss in multiple image classification datasets depending on the training round for Krum.}
    \label{fig:loss_no_attack}
\end{figure}

\begingroup
    \centering
    \begin{longtable}{llrrrrr}
    \toprule
    Dataset & Method & Final Test & Average Test & Min. Test & Average Test & Max. Test\\
        & & Loss& Loss & Loss & Accuracy & Accuracy \\
    \midrule
    \multirow{4}{*}{CIFAR-10} & Krum & 0.584 & 0.608 & 0.584 & 0.844 & 0.844 \\
     & Layerwise Krum & \textbf{0.505} & 0.509 & 0.457 & 0.865 & 0.870 \\
     & Cosine Krum & 0.603 & 0.598 & 0.578 & 0.836 & 0.843 \\
    & Layerwise Cosine Krum & 0.516 & \textbf{0.492} & \textbf{0.426} & \textbf{0.871} & \textbf{0.884} \\
    \midrule
     \multirow{4}{*}{CelebA-S}& Krum & 0.373 & 0.377 & 0.361 & 0.913 & 0.917 \\
     & Layerwise Krum & \textbf{0.352} & 0.370 & 0.291 & \textbf{0.920} & 0.920 \\
     & Cosine Krum & 0.385 & 0.388 & 0.364 & 0.912 & 0.915 \\
     & Layerwise Cosine Krum & 0.368 & \textbf{0.340} & \textbf{0.286} & 0.917 & \textbf{0.924} \\
    \midrule
    \multirow{4}{*}{CelebA-S} & Krum & 0.522 & 0.517 & 0.393 & 0.724 & 0.738 \\
      & Layerwise Krum & 0.460 & 0.654 & 0.373 & 0.712 & 0.731 \\
      & Cosine Krum & 0.475 & 0.469 & 0.363 & 0.870 & 0.895 \\
     Non-IID & Layerwise Cosine Krum & \textbf{0.408} & \textbf{0.426} & \textbf{0.317} & \textbf{0.885} & \textbf{0.907} \\
    \midrule
     \multirow{4}{*}{EMNIST}& Krum & 0.244 & 0.244 & 0.087 & 0.984 & 0.985 \\
     & Layerwise Krum & 0.062 & 0.064 & 0.059 & 0.987 & 0.988 \\
     & Cosine Krum & 0.062 & 0.066 & 0.053 & 0.989 & \textbf{0.990} \\
     & Layerwise Cosine Krum & \textbf{0.055} & \textbf{0.053} & \textbf{0.046} & \textbf{0.989} & \textbf{0.990} \\
    \midrule
    \multirow{4}{*}{EMNIST} & Krum & 1.734 & 1.663 & 0.761 & 0.899 & 0.907 \\
     & Layerwise Krum & 0.683 & 0.669 & 0.504 & 0.923 & 0.928 \\
     & Cosine Krum & 0.477 & 0.517 & 0.416 & 0.936 & \textbf{0.944} \\
     Non-IID & Layerwise Cosine Krum & \textbf{0.381} & \textbf{0.380} & \textbf{0.332} & \textbf{0.935} & 0.943 \\
    \midrule
    \multirow{4}{*}{Fashion}& Krum & 1.648 & 1.706 & 0.626 & 0.852 & 0.859 \\
     & Layerwise Krum & 0.613 & 0.604 & 0.491 & 0.870 & 0.874 \\
     & Cosine Krum & 0.754 & 0.637 & 0.468 & 0.878 & 0.880 \\
     MNIST & Layerwise Cosine Krum & \textbf{0.563} & \textbf{0.532} & \textbf{0.419} & \textbf{0.879} & \textbf{0.887} \\
    \bottomrule
    \caption{Test Loss and Accuracy for every method under no attack for Krum.\label{tab:loss}}
\end{longtable}

\endgroup

\begin{figure}[ht]
    \begin{subfigure}{0.30\linewidth}
        \centering
        \includegraphics[width=\linewidth]{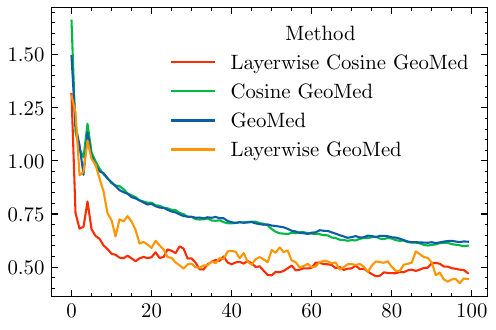}
        \caption{CIFAR-10 dataset.}
    \end{subfigure}
    \begin{subfigure}{0.30\linewidth}
        \centering
        \includegraphics[width=\linewidth]{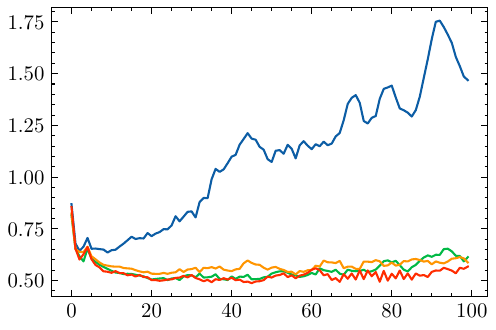}
        \caption{Fashion MNIST dataset.}
    \end{subfigure}
    \begin{subfigure}{0.30\linewidth}
        \centering
        \includegraphics[width=\linewidth]{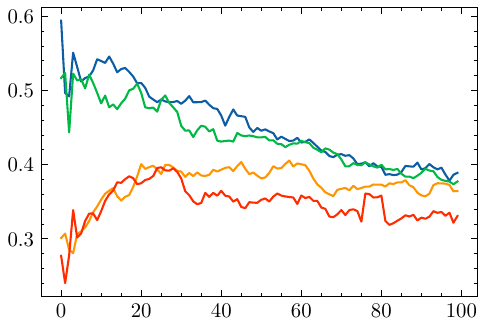}
        \caption{CelebA-S dataset.}
    \end{subfigure}
    \caption{Test Loss in multiple image classification datasets depending on the training round for GeoMed.}
    \label{fig:loss_no_attack_geomed}
    
\end{figure}

\begingroup
    \centering
    
    \begin{longtable}{llrrrrr}
\toprule
 Dataset & Method & Final Test & Average Test & Min. Test & Average Test & Max. Test\\
        & & Loss& Loss & Loss & Accuracy & Accuracy \\
\midrule
\multirow{4}{*}{CIFAR-10} & GeoMed & 0.621 & 0.621 & 0.594 & 0.837 & 0.840 \\
             & Layerwise GeoMed & 0.425 & \textbf{0.437} & \textbf{0.375} & \textbf{0.877} & \textbf{0.892} \\
             & Cosine GeoMed & 0.616 & 0.607 & 0.580 & 0.838 & 0.841 \\
             & Layerwise Cosine GeoMed & \textbf{0.418} & 0.487 & 0.418 & \textbf{0.877} & 0.889 \\
\midrule
\multirow{4}{*}{CelebA-S} & GeoMed & 0.401 & 0.391 & 0.361 & 0.913 & 0.916 \\
 & Layerwise GeoMed & 0.364 & 0.370 & 0.281 & 0.918 & 0.922 \\
 & Cosine GeoMed & 0.389 & 0.380 & 0.359 & 0.912 & 0.914 \\
 & Layerwise Cosine GeoMed & \textbf{0.351} & \textbf{0.333} & \textbf{0.240} & \textbf{0.921} & \textbf{0.924} \\
\midrule
\multirow{4}{*}{CelebA-S} & GeoMed & \textbf{0.409} & 0.528 & 0.382 & 0.871 & 0.891 \\
 & Layerwise GeoMed & 0.661 & 0.575 & \textbf{0.322} & 0.853 & 0.892 \\
 & Cosine GeoMed & 0.756 & 0.570 & 0.373 & 0.856 & 0.892 \\
 Non-IID & Layerwise Cosine GeoMed & 0.592 & \textbf{0.464} & 0.347 & \textbf{0.894} & \textbf{0.909} \\
\midrule
\multirow{4}{*}{EMNIST} & GeoMed & 0.229 & 0.234 & 0.100 & 0.983 & 0.984 \\
 & Layerwise GeoMed & 0.092 & 0.098 & 0.080 & 0.986 & 0.986 \\
 & Cosine GeoMed & 0.068 & 0.064 & 0.051 & 0.989 & 0.990 \\
 & Layerwise Cosine GeoMed & \textbf{0.055} & \textbf{0.050} & \textbf{0.046} & \textbf{0.990} &\textbf{ 0.991} \\
\midrule
\multirow{4}{*}{EMNIST} & GeoMed & 2.670 & 2.018 & 0.729 & 0.891 & 0.907 \\
 & Layerwise GeoMed & 0.725 & 0.643 & 0.476 & 0.931 & 0.936 \\
 & Cosine GeoMed & 0.659 & 0.576 & 0.431 & 0.931 & \textbf{0.939} \\
  Non-IID & Layerwise Cosine GeoMed & \textbf{0.325} & \textbf{0.391} & \textbf{0.325} &\textbf{ 0.933} & 0.937 \\
\midrule
\multirow{4}{*}{Fashion} & GeoMed & 1.496 & 1.579 & 0.604 & 0.855 & 0.860 \\
 & Layerwise GeoMed & 0.538 & 0.590 & 0.482 & 0.869 & 0.872 \\
 & Cosine GeoMed & 0.732 & 0.633 & 0.455 & 0.879 & 0.883 \\
 MNIST & Layerwise Cosine GeoMed & \textbf{0.533} & \textbf{0.561} & \textbf{0.411} & \textbf{0.883} & \textbf{0.887} \\
\bottomrule
    \caption{Test Loss and Accuracy for every method under no attack for GeoMed.\label{tab:loss_geomed}}
    \end{longtable}
    
\endgroup

\begin{figure}[ht]
    \begin{subfigure}{0.30\linewidth}
        \centering
        \includegraphics[width=\linewidth]{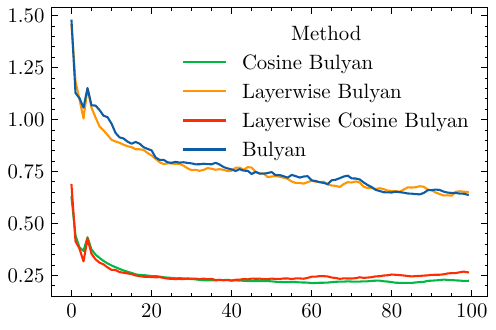}
        \caption{CIFAR-10 dataset.}
    \end{subfigure}
    \begin{subfigure}{0.30\linewidth}
        \centering
        \includegraphics[width=\linewidth]{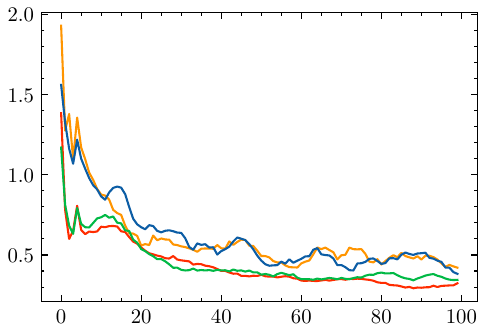}
        \caption{EMNIST Non-IID dataset.}
    \end{subfigure}
    \begin{subfigure}{0.30\linewidth}
        \centering
        \includegraphics[width=\linewidth]{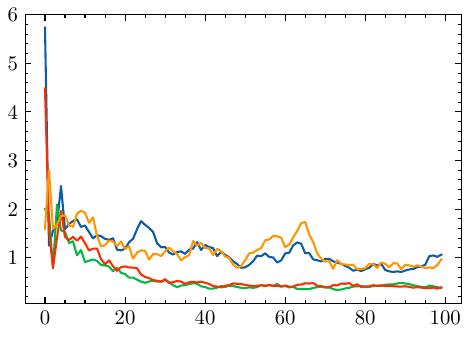}
        \caption{CelebA-S Non-IID dataset.}
    \end{subfigure}
    \caption{Test Loss in multiple image classification datasets depending on the training round for Bulyan.}
    \label{fig:loss_no_attack_bulyan}
\end{figure}

\begingroup
    \centering
    
    \begin{longtable}{llrrrrr}
    \toprule
    Dataset & Method & Final Test & Average Test & Min. Test & Average Test & Max. Test\\
        & & Loss& Loss & Loss & Accuracy & Accuracy \\
    \midrule
    \multirow{4}{*}{CIFAR-10} & Bulyan & 0.641 & 0.643 & 0.622 & 0.828 & 0.833 \\
     & Layerwise Bulyan & 0.628 & 0.644 & 0.616 & 0.828 & 0.831 \\
     & Cosine Bulyan & \textbf{0.230} & \textbf{0.226} & \textbf{0.210} & \textbf{0.946} & \textbf{0.948} \\
     & Layerwise Cosine Bulyan & 0.252 & 0.262 & 0.211 & \textbf{0.946} & \textbf{0.948 }\\
    \midrule
    \multirow{4}{*}{CelebA-S} & Bulyan & 0.405 & 0.388 & 0.367 & 0.910 & 0.915 \\
     & Layerwise Bulyan & 0.378 & 0.379 & 0.361 & 0.912 & 0.914 \\
     & Cosine Bulyan & \textbf{0.351} & \textbf{0.338} & \textbf{0.210} & 0.923 & 0.926 \\
     & Layerwise Cosine Bulyan & 0.396 & 0.384 & 0.221 & \textbf{0.924} & \textbf{0.927} \\
    \midrule
    \multirow{4}{*}{CelebA-S} & Bulyan & 0.952 & 0.940 & 0.535 & 0.810 & 0.865 \\
     & Layerwise Bulyan & 1.242 & 0.891 & 0.523 & 0.789 & 0.864 \\
     & Cosine Bulyan & \textbf{0.264} & \textbf{0.388} & \textbf{0.264} & 0.901 & 0.913 \\
    Non-IID & Layerwise Cosine Bulyan & 0.471 & \textbf{0.388} & 0.308 & \textbf{0.902} & \textbf{0.914} \\
    \midrule
    \multirow{4}{*}{EMNIST} & Bulyan & 0.057 & 0.057 & 0.051 & 0.989 & 0.990 \\
     & Layerwise Bulyan & 0.057 & 0.054 & 0.048 & 0.989 & 0.990 \\
     & Cosine Bulyan & 0.036 & \textbf{0.036} & \textbf{0.034} & \textbf{0.993} & \textbf{0.993} \\
     & Layerwise Cosine Bulyan & \textbf{0.034} & 0.038 & \textbf{0.034} & 0.992 & \textbf{0.993} \\
    \midrule
    \multirow{4}{*}{EMNIST} & Bulyan & 0.360 & 0.423 & 0.311 & 0.930 & 0.945 \\
     & Layerwise Bulyan & 0.411 & 0.446 & 0.341 & 0.930 & 0.942 \\
     & Cosine Bulyan & 0.366 & 0.358 & 0.326 & \textbf{0.955 }& \textbf{0.959} \\
      Non-IID & Layerwise Cosine Bulyan & \textbf{0.352} & \textbf{0.312} & \textbf{0.267} & 0.947 & 0.950 \\
    \midrule
    \multirow{4}{*}{Fashion} & Bulyan & 0.491 & 0.495 & 0.452 & 0.882 & 0.885 \\
     & Layerwise Bulyan & 0.460 & 0.497 & 0.433 & 0.881 & 0.885 \\
     & Cosine Bulyan & 0.473 & \textbf{0.448} & \textbf{0.374} & \textbf{0.905} & \textbf{0.907} \\
     MNIST & Layerwise Cosine Bulyan & \textbf{0.377} & 0.459 & 0.375 & 0.904 & 0.905 \\
    \bottomrule
    \caption{Test Loss and Accuracy for every method under no attack for Bulyan.\label{tab:loss_bulyan}}
\end{longtable}
    
\endgroup

\subsection{Analysis under Byzantine attack}\label{sec:underattack}
\revista{In the second scenario, where {FL} environments are subjected to adversarial attacks, the primary results are summarized in Tables \ref{tab:flip}, \ref{tab:flip_geomed} and \ref{tab:flip_bulyan}, showing the metrics for Krum, GeoMed and Bulyan respectively. Layerwise Cosine Aggregation, and its partial applications, consistently outperforms the baseline operators across all evaluated datasets. As depicted in Figures \ref{fig:underattack}, \ref{fig:underattack_geomed} and \ref{fig:underattack_bulyan} and by comparing these results tables in the previous section, our proposed method achieves performance comparable to that obtained in the absence of adversarial clients. This empirical evidence highlights the strong Byzantine robustness of Layerwise Cosine Aggregation.}

\revista{While a detailed analysis of these results is provided in the previous scenario, it is crucial to reiterate the importance of gradient clipping within Layerwise Cosine Aggregation. Because cosine distance is sensitive to direction but not magnitude, it may incorrectly identify two updates with drastically different norms as similar. This can potentially compromise convergence by selecting updates with excessively large norms. This issue is particularly pronounced in settings employing the \texttt{EfficientNet-B0} model, where high-norm updates in batch normalization layers can render the model ineffective. By incorporating median gradient clipping, we effectively focus on the gradient direction, mitigating this problem.}

\revista{Therefore, and consistent with the results from the previous scenario, our proposed aggregation operator emerges as a more robust and effective choice in the presence of adversarial attacks.}
\begin{figure}[H]
\begin{subfigure}{0.30\linewidth}
    \centering
    \includegraphics[width=\linewidth]{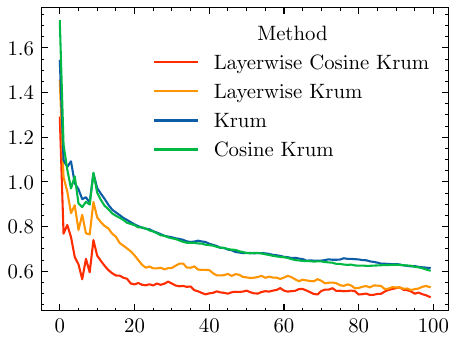}
    \caption{CIFAR-10 dataset.}
\end{subfigure}
\begin{subfigure}{0.30\linewidth}
    \centering
    \includegraphics[width=\linewidth]{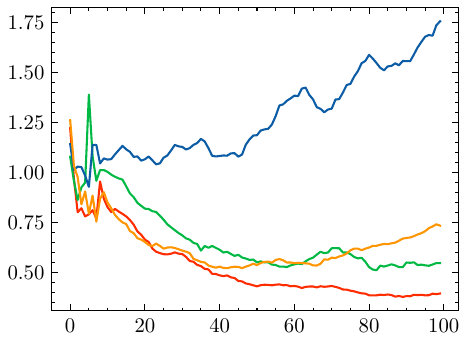}
    \caption{EMNIST Non-IID dataset.}
\end{subfigure}
\centering
\begin{subfigure}{0.30\linewidth}
    \centering
    \includegraphics[width=\linewidth]{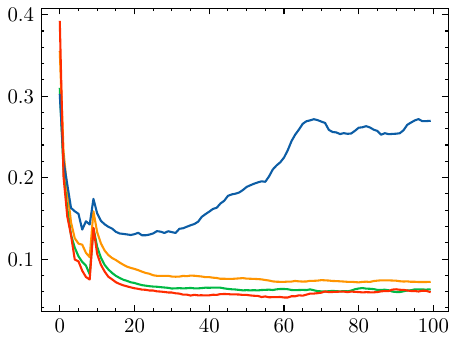}
    \caption{EMNIST dataset.}
\end{subfigure}
\caption{Test Loss under a Label Flipping attack depending on the training round for Krum.}
\label{fig:underattack}
\end{figure}

\begingroup
    \centering
    
    \begin{longtable}{llrrrrr}
    \toprule
        Dataset & Method & Final Test & Average Test & Min. Test & Average Test & Max. Test\\
        & & Loss& Loss & Loss & Accuracy & Accuracy \\
        \midrule
         \multirow{4}{*}{CIFAR-10}& Krum & 0.598 & 0.615 & 0.591 & 0.836 & 0.840 \\
         & Layerwise Krum & \textbf{0.496} & 0.530 & 0.463 & \textbf{0.860} & 0.871 \\
        & Cosine Krum & 0.577 & 0.603 & 0.577 & 0.838 & 0.844 \\
         & Layerwise Cosine Krum & 0.502 & \textbf{0.485} & \textbf{0.420} & 0.846 & \textbf{0.888} \\
        \midrule
         \multirow{4}{*}{CelebA-S}& Krum & 0.379 & 0.389 & 0.365 & 0.914 & 0.917 \\
         & Layerwise Krum & 0.364 & 0.379 & 0.289 & 0.917 & 0.921 \\
        & Cosine Krum & 0.402 & 0.384 & 0.345 & 0.911 & 0.915 \\
         & Layerwise Cosine Krum & \textbf{0.334} & \textbf{0.326} & \textbf{0.245} & \textbf{0.920} & \textbf{0.923} \\
        \midrule
         \multirow{4}{*}{CelebA-S}& Krum & \textbf{0.491} & 0.548 & 0.425 & 0.723 & 0.736 \\
         & Layerwise Krum & 0.556 & 0.668 & 0.353 & 0.841 & 0.887 \\
        & Cosine Krum & 0.772 & 0.548 & 0.350 & 0.857 & 0.896 \\
        Non-IID & Layerwise Cosine Krum & 0.531 & \textbf{0.456} & \textbf{0.338} & \textbf{0.887} & \textbf{0.902} \\
        \midrule
         \multirow{4}{*}{EMNIST}& Krum & 0.268 & 0.270 & 0.115 & 0.982 & 0.983 \\
         & Layerwise Krum & 0.074 & 0.072 & 0.066 & 0.987 & 0.987 \\
        & Cosine Krum & \textbf{0.058} & 0.063 & 0.052 & \textbf{0.989} & \textbf{0.990} \\
         & Layerwise Cosine Krum & \textbf{0.058} & \textbf{0.060} & \textbf{0.047} & 0.988 & 0.989 \\
        \midrule
         \multirow{4}{*}{EMNIST}& Krum & 1.911 & 1.756 & 0.756 & 0.903 & 0.911 \\
         & Layerwise Krum & 0.708 & 0.733 & 0.461 & 0.927 & 0.931 \\
        & Cosine Krum & 0.449 & 0.545 & 0.437 & 0.932 & \textbf{0.944} \\
        Non-IID & Layerwise Cosine Krum & \textbf{0.400} & \textbf{0.393} & \textbf{0.322} & \textbf{0.933} & 0.938 \\
        \midrule
         \multirow{4}{*}{Fashion} & Krum & 1.434 & 1.764 & 0.641 &  0.849 & 0.854 \\
         & Layerwise Krum & 0.642 & 0.653 & 0.492 & 0.865 & 0.869 \\
         & Cosine Krum & 0.681 & 0.644 & 0.484 & \textbf{0.874} & 0.879 \\
        MNIST & Layerwise Cosine Krum & \textbf{0.633} & \textbf{0.555} & \textbf{0.405} & 0.870 & \textbf{0.880} \\
        \bottomrule
    \caption{Test Loss and Accuracy for every method under the label flipping attack for Krum.    \label{tab:flip}}
        \end{longtable}

\endgroup

\begin{figure}[H]
\begin{subfigure}{0.30\linewidth}
    \centering
    \includegraphics[width=\linewidth]{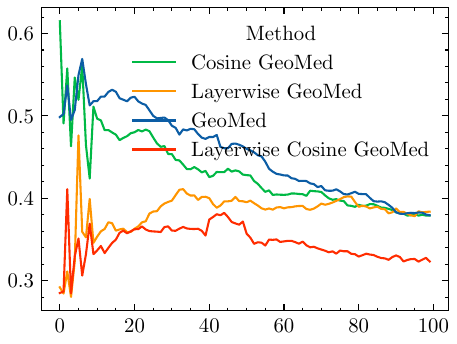}
    \caption{Celeba-S dataset.}
\end{subfigure}
\begin{subfigure}{0.30\linewidth}
    \centering
    \includegraphics[width=\linewidth]{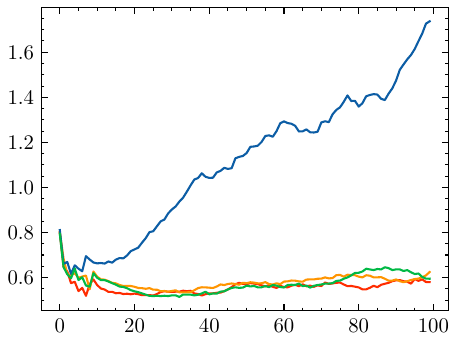}
    \caption{Fashion MNIST dataset.}
\end{subfigure}
\centering
\begin{subfigure}{0.30\linewidth}
    \centering
    \includegraphics[width=\linewidth]{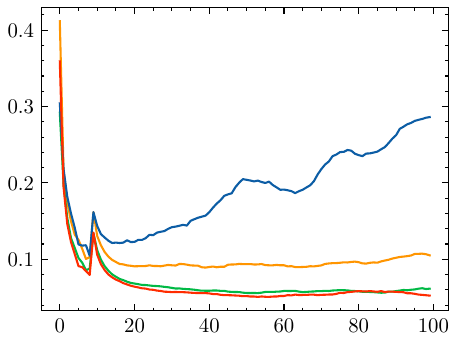}
    \caption{EMNIST dataset.}
\end{subfigure}
\caption{Test Loss under a Label Flipping attack depending on the training round for GeoMed.}
\label{fig:underattack_geomed}
\end{figure}

\begingroup
    \centering
    
   \begin{longtable}{llrrrrr}
    \toprule
    Dataset & Method & Final Test & Average Test & Min. Test & Average Test & Max. Test\\
        & & Loss& Loss & Loss & Accuracy & Accuracy \\
    \midrule
    \multirow{4}{*}{CIFAR-10} & GeoMed & 0.598 & 0.607 & 0.587 & 0.838 & 0.841 \\
    & Layerwise GeoMed & 0.585 &\textbf{ 0.499} & \textbf{0.399} & 0.865 & 0.881 \\
    & Cosine GeoMed & 0.622 & 0.602 & 0.588 & 0.839 & 0.841 \\
    & Layerwise Cosine GeoMed & \textbf{0.534} & 0.509 & 0.408 & \textbf{0.872} & \textbf{0.888} \\
    \midrule
    \multirow{4}{*}{CelebA-S} & GeoMed & 0.374 & 0.379 & 0.356 & 0.913 & 0.917 \\
    & Layerwise GeoMed & 0.385 & 0.384 & \textbf{0.280} & 0.918 & 0.919 \\
    & Cosine GeoMed & 0.357 & 0.379 & 0.357 & 0.912 & 0.916 \\
    & Layerwise Cosine GeoMed &\textbf{ 0.316} & \textbf{0.323} & 0.285 & \textbf{0.921} &\textbf{ 0.924} \\
    \midrule
    \multirow{4}{*}{CelebA-S} & GeoMed & 0.563 & 0.528 & 0.382 & 0.876 & 0.897 \\
    & Layerwise GeoMed & \textbf{0.304} & 0.846 & \textbf{0.304} & 0.805 & 0.890 \\
    & Cosine GeoMed & 0.493 & 0.571 & 0.386 & 0.858 & 0.890 \\
   Non-IID & Layerwise Cosine GeoMed & 0.456 & \textbf{0.444} & 0.319& \textbf{0.897} & \textbf{0.911} \\
    \midrule
    \multirow{4}{*}{EMNIST} & GeoMed & 0.297 & 0.286 & 0.105 & 0.982 & 0.983 \\
    & Layerwise GeoMed & 0.095 & 0.105 & 0.078 & 0.985 & 0.986 \\
    & Cosine GeoMed & 0.061 & 0.062 & 0.052 & \textbf{0.989} & \textbf{0.990} \\
    & Layerwise Cosine GeoMed & \textbf{0.051} & \textbf{0.052} & \textbf{0.046} & \textbf{0.989} & \textbf{0.990} \\
    \midrule
    \multirow{4}{*}{EMNIST} & GeoMed & 1.165 & 1.654 & 0.772 & 0.897 & 0.910 \\
    & Layerwise GeoMed & 0.555 & 0.598 & 0.476 & 0.927 & 0.937 \\
    & Cosine GeoMed & 0.708 & 0.544 & 0.458 & \textbf{0.933} & \textbf{0.943} \\
   Non-IID & Layerwise Cosine GeoMed & \textbf{0.378} &\textbf{ 0.376} & \textbf{0.338} & \textbf{0.933} & 0.940 \\
    \midrule
    \multirow{4}{*}{Fashion} & GeoMed & 1.707 & 1.737 & 0.600 & 0.851 & 0.858 \\
    & Layerwise GeoMed & 0.758 & 0.625 & 0.500 & 0.867 & 0.871 \\
    & Cosine GeoMed & \textbf{0.564} & 0.595 & 0.458 & \textbf{0.876} & \textbf{0.880} \\
   MNIST & Layerwise Cosine GeoMed & 0.663 & \textbf{0.580} & \textbf{0.450} & 0.868 & 0.878 \\
    \bottomrule
    \caption{Test Loss and Accuracy for every method under the label flipping attack for GeoMed.\label{tab:flip_geomed}}
    \end{longtable}

\endgroup

\begin{figure}[H]
\begin{subfigure}{0.30\linewidth}
    \centering
    \includegraphics[width=\linewidth]{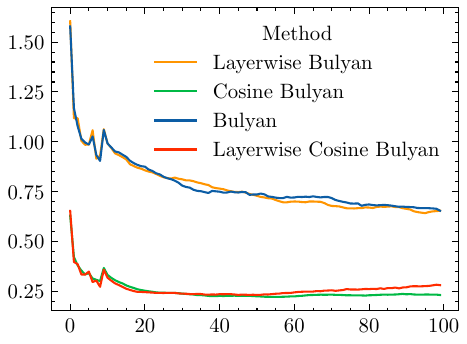}
    \caption{CIFAR-10 dataset.}
\end{subfigure}
\begin{subfigure}{0.30\linewidth}
    \centering
    \includegraphics[width=\linewidth]{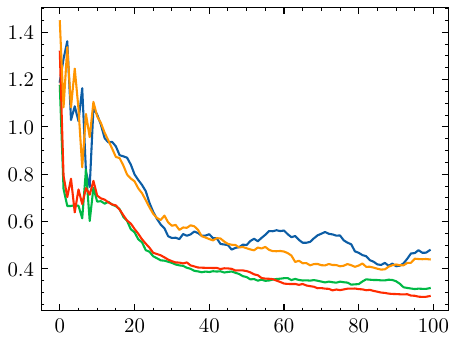}
    \caption{EMNIST Non-IID dataset.}
\end{subfigure}
\centering
\begin{subfigure}{0.30\linewidth}
    \centering
    \includegraphics[width=\linewidth]{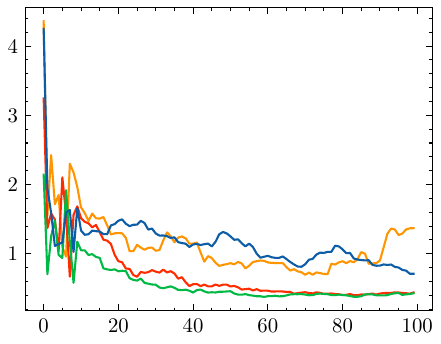}
    \caption{Celeba-S Non-IID dataset.}
\end{subfigure}
\caption{Test Loss under a Label Flipping attack depending on the training round for Bulyan.}
\label{fig:underattack_bulyan}
\end{figure}

\begingroup
    \small
    \centering
    
    \begin{longtable}{llrrrrr}
    \toprule
     Dataset & Method & Final Test & Average Test & Min. Test & Average Test & Max. Test\\
        & & Loss& Loss & Loss & Accuracy & Accuracy \\
    \midrule
    \multirow{4}{*}{CIFAR-10} & Bulyan & 0.648 & 0.655 & 0.637 & 0.826 & 0.829 \\
    & Layerwise Bulyan & 0.655 & 0.655 & 0.626 & 0.827 & 0.831 \\
    & Cosine Bulyan & \textbf{0.229} & \textbf{0.234} & \textbf{0.217} & \textbf{0.946} & \textbf{0.948} \\
    & Layerwise Cosine Bulyan & 0.287 & 0.283 & 0.222 & 0.943 & 0.947 \\
    \midrule
    \multirow{4}{*}{CelebA-S} & Bulyan & 0.373 & 0.386 & 0.364 & 0.910 & 0.915 \\
    & Layerwise Bulyan & 0.410 & 0.411 & 0.371 & 0.911 & 0.914 \\
    & Cosine Bulyan & \textbf{0.369} & \textbf{0.353} & \textbf{0.220} & \textbf{0.923} &\textbf{ 0.926} \\
    & Layerwise Cosine Bulyan & 0.383 & 0.388 & 0.232 & 0.921 & \textbf{0.926} \\
    \midrule
    \multirow{4}{*}{CelebA-S} & Bulyan & 0.686 & 0.711 & 0.443 & 0.829 & 0.869 \\
    & Layerwise Bulyan & 0.942 & 1.369 & 0.476 & 0.757 & 0.861 \\
    & Cosine Bulyan & \textbf{0.415} & \textbf{0.429 }& \textbf{0.260} & \textbf{0.904} & 0.915 \\
   Non-IID & Layerwise Cosine Bulyan & 0.513 & 0.440 & 0.322 & 0.899 & \textbf{0.916} \\
    \midrule
    \multirow{4}{*}{EMNIST} & Bulyan & 0.042 & 0.056 & 0.042 & 0.989 & 0.990 \\
    & Layerwise Bulyan & 0.060 & 0.059 & 0.051 & 0.989 & 0.990 \\
    & Cosine Bulyan & \textbf{0.034} &\textbf{ 0.037} & \textbf{0.034} & \textbf{0.993} & \textbf{0.993} \\
    & Layerwise Cosine Bulyan & 0.039 & 0.039 & \textbf{0.034} & 0.992 & \textbf{0.993} \\
    \midrule
    \multirow{4}{*}{EMNIST} & Bulyan & 0.570 & 0.479 & 0.316 & 0.929 & 0.941 \\
    & Layerwise Bulyan & 0.443 & 0.440 & 0.321 & 0.932 & 0.942 \\
    & Cosine Bulyan & 0.373 & 0.319 & 0.292 & \textbf{0.959} & \textbf{0.962} \\
   Non-IID & Layerwise Cosine Bulyan & \textbf{0.298} & \textbf{0.285} & \textbf{0.236} & 0.950 & 0.955 \\
    \midrule
    \multirow{4}{*}{Fashion} & Bulyan & 0.536 & 0.501 & 0.455 & 0.882 & 0.885 \\
    & Layerwise Bulyan & 0.481 & 0.476 & 0.428 & 0.883 & 0.886 \\
    & Cosine Bulyan & 0.568 & 0.495 & 0.384 & \textbf{0.901} & \textbf{0.902} \\
    MNIST & Layerwise Cosine Bulyan & \textbf{0.438} & \textbf{0.488} & \textbf{0.378} & 0.899 & \textbf{0.902} \\
    \bottomrule
    \caption{Test Loss and Accuracy for every method under the label flipping attack for Bulyan.\label{tab:flip_bulyan}}
    \end{longtable}
\endgroup

\section{Conclusions and Future Work}\label{sec:conclusions}
\revista{Due to the distributed nature of FL, mitigating Byzantine attacks has become a critical area of research. The introduction of $(\alpha, f)$-Byzantine resilient operators represented a significant advancement in the field, both theoretically and empirically. However, with continued use, its limitations have become apparent. This work addresses these limitations by exploring the theoretical properties and proposing Layerwise Cosine {aggregation rule}, an improved aggregation scheme that introduces minimal computational overhead. The principal contributions of this work are as follows:}
\begin{itemize}
    \item Show, both theoretically and empirically, a discernible decrease in performance for robust aggregation rules based on $(\alpha, f)$-Byzantine resilience when applied to high-dimensional data. This observation reveals a significant performance gap between robust and effective FL models in such settings, highlighting a critical challenge in the field.
    \item To propose the layer-wise aggregation, while preserving the theoretical properties of Byzantine resilience, allows us to derive a tighter theoretical bound on the expected angle, leading to improved robustness against Byzantine attacks.
    \item To combine layer-wise aggregation with the use of cosine distance, we obtain a theoretically and empirically significantly improved robust aggregation scheme for certain operators: Layerwise Cosine Aggregation.
\end{itemize}

Our findings demonstrate the broad applicability of our proposed scheme, exemplified by its successful enhancement of widely adopted operators such as Krum, Bulyan, and GeoMed. Across all scenarios studied, our approach consistently yielded performance improvements, indicating a promising step toward bridging robust and effective privacy-conscious training paradigms. This work also highlights the potential for revisiting the foundations of well-established methods, enabling significant improvements across multiple related areas. 

\textbf{Future Work.} Although this study has demonstrated the effectiveness of Layerwise Cosine Aggregation in improving robustness under Byzantine scenarios within image classification tasks, several promising directions remain open for exploration. Future work may extend this framework to other data modalities such as natural language processing and time series forecasting, where high-dimensional and sparse representations are common. Furthermore, deploying and evaluating the method in real-world federated environments, where device availability, communication constraints, and heterogeneous hardware introduce further complexity, would offer valuable practical insights. Finally, exploring adaptive or learnable layer-wise distance metrics, beyond static cosine similarity, could further enhance the flexibility and effectiveness of the aggregation process across diverse neural architectures.


\section*{Acknowledgments}
This research results of the Strategic Project IAFER-Cib (C074/23), as a result of the collaboration agreement signed between the National Institute of Cybersecurity (INCIBE) and the University of Granada. This initiative is carried out within the framework of the Recovery, Transformation, and Resilience Plan funds, financed by the European Union (Next Generation).
\bibliographystyle{unsrt}  
\bibliography{references}

\end{document}